\documentclass[a4paper, 10pt, conference]{IEEEtran}      

\IEEEoverridecommandlockouts                              
                                                          
\title{\LARGE \bf
Synthesis of Provably Correct Autonomy Protocols\\ for Shared Control
}

\author{Murat Cubuktepe, Nils Jansen, Mohammed Alsiekh, Ufuk Topcu\thanks{M. Cubuktepe and U. Topcu are with the Department of Aerospace Engineering
and Engineering Mechanics, University of Texas at Austin, 201 E 24th St, Austin, TX 78712, USA. Nils Jansen is with the Department of Software Science, Radboud University Nijmegen, Comeniuslaan 4, 6525 HP Nijmegen, the Netherlands. Mohammed Alsiekh was with the Institute for Computational
Engineering and Sciences, University of Texas at Austin, 201 E 24th
St, Austin, TX 78712, USA. email:(\{mcubuktepe,malsiekh,utopcu\}@utexas.edu, n.jansen@science.ru.nl).}
}

\usepackage{etex}
\usepackage{amsmath,amssymb,amsfonts}
\usepackage{mathtools}
\usepackage{amsthm}
\usepackage[usenames,dvipsnames]{color}
\usepackage{xcolor}
\usepackage{xspace}
\usepackage{microtype}
\usepackage{todonotes}
\usepackage{colonequals}
\usepackage{wasysym}
 \usepackage{booktabs}
\usepackage{mathtools}
\DeclarePairedDelimiter\ceil{\lceil}{\rceil}

\usepackage[ruled,vlined]{algorithm2e}
\usepackage{multirow}
\usepackage{multicol}
\usepackage{balance}

\usepackage{paralist}

\usepackage{subfigure}
\usepackage{url}
\usepackage{graphicx}
\usepackage{float}
\usepackage{listings}
\usepackage{nicefrac}
\usepackage{tikz} 
\usepackage{pgfplots}
\pgfplotsset{compat=1.8}
\usepackage{marvosym}
\usepackage{wrapfig}

\usetikzlibrary{arrows,decorations.pathmorphing,positioning,fit,trees,shapes,shadows,automata,calc} 

\usepackage{macros} 

\tikzset{outline/.style args={#1}{%
  draw=#1,thick,fill=#1!50}}

\begin{document}

\maketitle
\thispagestyle{empty}
\pagestyle{empty}

\begin{abstract}
We synthesize shared control protocols subject to probabilistic temporal logic specifications.
More specifically, we develop a framework in which a human and an autonomy protocol can issue commands to carry out a certain task. 
We blend these commands into a joint input to a robot. 
We model the interaction between the human and the robot as a Markov decision process (MDP) that represents the shared control scenario. 
Using inverse reinforcement learning, we obtain an abstraction of the human's behavior and decisions.
We use randomized strategies to account for randomness in human's decisions, caused by factors such as complexity of the task specifications or imperfect interfaces. 
We design the autonomy protocol to ensure that the resulting robot behavior satisfies given safety and performance specifications in probabilistic temporal logic.
Additionally, the resulting strategies generate behavior as similar to the behavior induced by the human's commands as possible. 
We solve the underlying problem efficiently using quasiconvex programming.
Case studies involving autonomous wheelchair navigation and unmanned aerial vehicle mission planning showcase the applicability of our approach.
\end{abstract}

\newif\ifufuk
\ufuktrue 
\newcommand{\ufuksubsection}[1]{\ifufuk\subsection{#1}\fi}
\newif\ifnocomments
\nocommentsfalse 
\ifnocomments 
\renewcommand{\nj}[1]{}
\renewcommand{\lmc}[1]{}
\fi

\section{Introduction}\label{sec:intro}
In shared control, a robot executes a task to accomplish the goals of a human operator while adhering to additional safety and performance requirements.
Applications of such human-robot interaction include remotely operated semi-autonomous wheelchairs~\cite{wheelchair-demillan}, robotic teleoperation~\cite{javdani2015shared}, and human-in-the-loop unmanned aerial vehicle mission planning~\cite{feng2016synthesis}. 
A human operator issues a command through an input interface, which maps the command directly to an action for the robot. 
The problem is that a sequence of such actions may fail to accomplish the task at hand, due to limitations of the interface or failure of the human operator in comprehending the complexity of the problem.
Therefore, a so-called \emph{autonomy protocol} provides assistance for the human in order to complete the task according to the given requirements.





At the heart of the shared control problem is the design of an autonomy protocol.
In the literature, there are two main directions, based on either \emph{switching} the control authority between human and autonomy protocol~\cite{shen2004collaborative}, or on \emph{blending} their commands towards joined inputs for the robot~\cite{dragan-et-al-policy-blending,jansen2017synthesis}. 

One approach to switching the authority first determines the desired goal of the human operator with high confidence, and then assists towards exactly this goal~\cite{fagg2004extracting,kofman2005teleoperation}.
In~\cite{DBLP:journals/tase/FuT16}, switching the control authority between the human and autonomy protocol ensures satisfaction of specifications that are formally expressed in temporal logic. 
In general, switching of authority may cause a decrease in human's satisfaction, who usually prefers to retain as much control as possible~\cite{kim2012autonomy}. 

Blending incorporates providing an alternative command in addition to the one of the human operator.
To introduce a more flexible trade-off between the human's control authority and the level of autonomous assistance,  both commands are then blended to form a joined input for the robot.
A \emph{blending function} determines the emphasis that is put on the autonomy protocol in the blending, that is, regulating the amount of assistance provided to the human.
Switching of authority can be seen as a special case of blending, as the blending function may assign full control to the autonomy protocol or to the human.
In general, putting more emphasis on the autonomy protocol in blending may lead to greater accuracy in accomplishing the task~\cite{dragan-et-al-assistive-teleoperation,dragan-et-al-policy-blending,leeper2012strategies}.
However, as humans prefer to retain control of the robot and may not approve if a robot issues a set of commands that is significantly different to the human's command~\cite{javdani2015shared,kim2012autonomy}.
None of the existing blending approaches provide \emph{formal correctness} guarantees that go beyond statistical confidence bounds.
Correctness here refers to ensuring safety and optimizing performance according to the given requirements.
Our goal is to design an autonomy protocol that admits formal correctness while rendering the robot behavior as close to the human's commands as possible, which is shown to enhance the human experience.

A human may be uncertain about which command to issue in order to  accomplish a task.
Moreover, a typical interface used to parse human's commands, such as a brain-computer interface, is inherently imperfect.
To capture such uncertainties and imperfections in the human's decisions, we introduce \emph{randomness} to the commands issued by humans.
It may not be possible to blend two different deterministic commands. 
If the human's command is ``up'' and the autonomy protocol's command is ``right'', we cannot blend these two commands to obtain another deterministic command.
By introducing randomness to the commands of the human and the autonomy protocol, we ensure that the blending is always well-defined.

Take as an example a scenario involving a semi-autonomous wheelchair~\cite{wheelchair-demillan} whose navigation has to account for a randomly moving autonomous vacuum cleaner, see Fig.~\ref{fig:gridworld}. 
The wheelchair needs to navigate to the exit of a room, and the vacuum cleaner moves in the room according to a probabilistic transition function.  
The task of the wheelchair is to reach the exit gate while not crashing into the vacuum cleaner. 
The human  may not fully perceive the motion of the vacuum cleaner. 
Note that the human's commands, depicted with the solid red line in Fig~\ref{fig:gridworld_human}, may cause the wheelchair to crash into the vacuum cleaner. 
The autonomy protocol provides another set of commands, which is indicated by the solid red line in Fig~\ref{fig:gridworld_full}, to carry out the task safely without crashing.
However, the autonomy protocol's commands deviate highly from the commands of the human. 
The two sets of commands are then blended into a new set of commands, depicted using the dashed red line in Fig~\ref{fig:gridworld_full}. 
The blended commands perform the task safely while generating behavior as similar to the behavior induced by the human's commands as possible.

We model the behavior of the robot as a Markov decision process (MDP)~\cite{puterman2014markov}, which captures the robot's actions inside a potentially stochastic environment.
%
%
%
%
Problem formulations with MDPs typically focus on maximizing an expected reward (or, minimizing the expected cost).
%
%
However, such formulations may not be sufficient to ensure safety or performance guarantees in a task that includes a human operator. 
Recently, it was shown that a reward structure is not sufficient to capture temporal logic constraints in general~\cite{hahn2019omega}.
We design the autonomy protocol such that the resulting robot behavior satisfies probabilistic temporal logic specifications. Such verification problems have been extensively studied for MDPs~\cite{BK08} and mature tools exist for efficient verification~\cite{KNP11,dehnert2017storm}.


\newcommand{\gridScale}{1} 

\newcommand\mygrid[1]{
    \draw[black,line width=1pt] (0,0) grid[step=1] (#1,#1);
    \draw[black,line width=4pt] (0,0) rectangle (#1,#1);
}

\newcommand{\fillGridAt}[3]{
	\node [xshift=.5*\gridScale cm,yshift=.5*\gridScale cm] at (#1,#2){#3};	
}

\renewcommand{\gridScale}{1}
\newcommand{\robotScale}{0.03}
\newcommand{\flagScale}{0.1}
\newcommand{\broomScale}{0.05}

\begin{figure}[t]
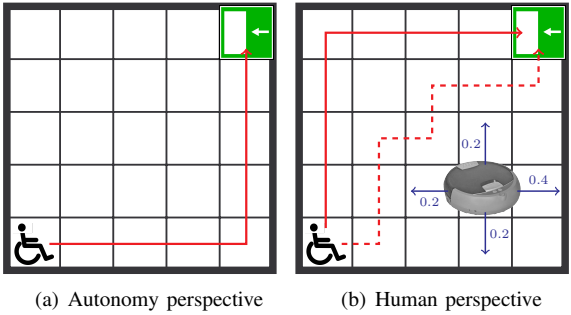

	\centering
		\subfigure[Autonomy perspective]
	{%
		\scalebox{0.7}{\input{wheelchair_human}}		
		\label{fig:gridworld_human}
	}%
	\subfigure[Human perspective]
	{%
		\scalebox{0.7}{\input{wheelchair_full}}
		\label{fig:gridworld_full}
	}%

	\caption{A wheelchair in a shared control setting.}
	\label{fig:gridworld}
\end{figure}
%

In what follows, we call a formal interpretation of a sequence of the human's commands the \emph{human strategy}, and the sequence of commands issued by the autonomy protocol the \emph{autonomy strategy}.
In \cite{jansen2017synthesis}, we formulated the problem of designing the autonomy protocol as a \emph{nonlinear programming problem}. 
However, solving nonlinear programs is generally intractable~\cite{bellare1993complexity}. 
Therefore, we proposed a greedy algorithm that iteratively \emph{repairs} the human strategy such that the specifications are satisfied without guaranteeing optimality, based on~\cite{pathak-et-al-nfm-2015}. 
Here, we propose an alternative approach for the blending of the two strategies. 
We follow the approach of repairing the strategy of the human to compute an autonomy protocol. We ensure that the resulting robot behavior induced by the repaired strategy deviates minimally from the human strategy, and satisfies safety and performance properties given in temporal logic specifications.
 We formally define the problem as a \emph{quasiconvex optimization problem}, which can be solved efficiently by checking feasibility of a number of convex optimization problems~\cite{boyd_convex_optimization}.

The question remains how to obtain the human strategy in the first place.
It may be unrealistic that a human can provide the strategy for an MDP that models a realistic scenario.
To this end, we create a virtual simulation environment that captures the behavior of the MDP. 
We ask humans to participate in two case studies to collect data about typical human behavior. 
We use \emph{inverse reinforcement learning} to get a formal interpretation as a strategy based on human's inputs~\cite{abbeel2004apprenticeship,ziebart2008maximum}.
We model a typical shared control scenario based on an autonomous wheelchair navigation~\cite{wheelchair-demillan} in our first case study. 
In our second case study, we consider an unmanned aerial vehicle mission planning scenario, where the human operator is to patrol certain regions while keeping away from enemy aerial vehicles.

In summary, the main contribution this paper is to efficiently synthesize an autonomy protocol such that the resulting blended or repaired strategy meets all given specifications while only minimally deviating from the human strategy. We present a new technique based on quasiconvex programming, which can be solved efficiently using convex optimization~\cite{boyd_convex_optimization}.

\textbf{Organization.} We introduce all formal foundations that we need in Section~\ref{sec:preliminaries}. We provide an overview of the general shared control concept in Section~\ref{sec:overview}. 
We present the \emph{shared control synthesis problem} and provide a solution based on convex optimization in Section~\ref{sec:synthesissec}. 
We indicate the applicability and scalability of our approach on experiments in Section~\ref{sec:simulation} and draw a conclusion and critique of our approach in Section~\ref{sec:conclusion}.

\section{Preliminaries}\label{sec:preliminaries}
In this section, we introduce the required formal models and specifications that we use to synthesize the autonomy protocol, and we give a short example illustrating the main concepts.

\subsection{Markov Decision Processes}
A \emph{probability distribution} over a finite set $\distDom$ is a function $\distFunc\colon\distDom\rightarrow\Ireal$ with $\sum_{\distDomElem\in\distDom}\distFunc(\distDomElem)=\distFunc(\distDom)=1$. 
The set $\distDom$ of all distributions is $\Distr(\distDom)$. 

\begin{definition}[MDP]
A \emph{Markov decision process (MDP)} $\MdpInit$ has a finite set $S$ of states, an initial state $\sinit \in S$, a finite set $\Act$ of actions, a transition probability function $\probmdp\colon S\times\Act\rightarrow\Distr(S)$, a finite set $\Ap$ of atomic propositions, and a labeling function $\Label \colon S \rightarrow 2^{\Ap}$ that labels each state $s \in S$ with a subset of atomic propositions $\Label (s) \subseteq \Ap$. We extend $\Label$ to a sequence of states by $\Label(s_0 s_1\ldots)=\Label(s_0)\Label(s_1)\ldots$ for $s_0,s_1 \in S$. 
\end{definition}
MDPs have \emph{nondeterministic choices} of actions at the states; the successors are determined \emph{probabilistically} via the associated probability distribution.
We assume that the MDP contains no deadlock states, that is, at every state at least one action is available.
A \emph{cost function} $C\colon S\times\Act\rightarrow\R_{\geq 0}$ associates a cost to state-action pairs.
%
%
If there is only a single action available at each state, the MDP reduces to a
\emph{discrete-time Markov chain (MC)}.
%
We use \emph{strategies} to resolve the choices of actions in order to define a probability and expected cost measure for MDPs.
%
\begin{definition}[Strategy]\label{def:scheduler}
	A \emph{memoryless and randomized strategy} for an MDP $\mdp$ is a function $\sched\colon S\rightarrow\Distr(\Act)$.
The set of all strategies over $\mdp$ is $\Sched^\mdp$.
\end{definition}
A special case are deterministic strategies which are functions of the form $\sched\colon S\rightarrow\Act$ with $\sigma(s) \in \Act(s)$.
Resolving all the nondeterminism for an MDP $\mdp$ with a strategy $\sched\in\Sched^\mdp$ yields an \emph{induced Markov chain} $\mdp^\sched$. 

\begin{definition}[Induced MC]\label{def:induced_dtmc}
	For an  MDP $\MdpInit$ and strategy $\sched\in\Sched^{\mdp}$, the \emph{MC induced by $\mdp$ and $\sched$} is  $\mdp^\sched=(S,\sinit,\Act,\probmdp^\sched)$, where
	\begin{align*}
		\probmdp^\sched(s,s')=\sum_{\act\in\Act(s)} \sched(s,\act)\cdot\probmdp(s,\act,s') \mbox{ for all } s,s'\in S.
	\end{align*} 
\end{definition}

In following, we assume that for a given MDP $\mathcal{M}$ and for any state $s$ $\in$ $S$, there exists a strategy $\sched$ that induces a MC $\mathcal{M}^{\sched}$ that ensures state $s$ is reachable under that strategy. Note that this is a standard assumption for MDPs, and we can remove the unreachable states by doing a graph search over the MDP $\mathcal{M}$ as a  preprocessing step~\cite{BK08}.

A finite or infinite sequence $\varrho^{\sched} = s_0s_1s_2\ldots$ of states generated in $\mathcal{M}$ under a strategy $\sched$ $\in$ $\Sched^{\mathcal{M}}$ is called a \textit{path}. Given an induced MC $\mathcal{M}^{\sched}$, starting from the initial state $s_0$, the state visited at step $t$ is given by a random variable $X_t$. The probability of reaching state $s'$ from state $s$ in one step, denoted $\mathbb{P}(X_{t+1}=s' | X_t = s)$ is equal to $\mathcal{P}^{\sched}(s,s')$. We can extend one-step reachability over a set of paths $\varrho^{\sched}$, i.e., $\mathbb{P}(s_{0} s_{1} s_{2} \ldots s_{n} )=\mathbb{P}(X_{n} = s_{n} | X_{n-1} = s_{n-1} )\cdot \mathbb{P}(s_{0} s_{1} s_{2} \ldots s_{n-1})$. We denote the set of all paths in $\mathcal{M}$ under the strategy $\sched$ by $Path^{\sched}(\mathcal{M})$.

%

\begin{definition}[Occupancy Measure]
The occupancy measure $x_\sched$ of a strategy $\sched$ for an MDP $\mdp$ is defined as 

\begin{align}
\displaystyle x_{\sched}(s,\act)=\mathbb{E}\left[ \sum_{t=0}^{\infty}P(\act_t=\act \vert s_t=s)\right],
\end{align}
where $s_t$ and $\act_t$ denote the state and action in $\mdp$ at time step $t$. The occupancy measure $x_{\sched}(s,\act)$ is the expected number of times to take action $\act$ at state $s$ under the strategy $\sched$.
\end{definition}
 In our solution approach, we use the occupancy measure of a strategy to compute an autonomy protocol.
%
%
%
%
\subsection{Specifications}

We use linear temporal logic (LTL) to specify a set of tasks~\cite{BK08}. A specification in LTL is built from a set $\Ap$ of atomic propositions, $\texttt{true},\texttt{false}$ and the Boolean and temporal connectives $\wedge,\vee, \neg, \Rightarrow, \Leftrightarrow$, and $\square$ (always), $\pctlUntil$ (until), $\eventually$ (eventually), and $\bigcirc $ (next). An infinite sequence of subsets of $\Ap$ defines an infinite word, and an LTL specification is interpreted over infinite words on $2^{\Ap}$. If a word $w=w_0 w_1 w_2\ldots$ satisfies an LTL specification $\varphi$, we denote it by $w \models \varphi$. 

\begin{definition}(DRA)
A deterministic Rabin automaton (DRA) is a tuple $\DraInit$, with a finite set $Q$ of states, an initial state $\qinit \in Q$, the alphabet $\Sigma$, the transition relation  $\mu \colon Q \times \Sigma \rightarrow Q$ between states of a DRA, and the set of accepting state pairs $\Acc \subseteq 2^{Q} \times 2^{Q}$.
\end{definition}

A run of a DRA $\mathcal{A}$, denoted by $\gamma=q_0 q_1 q_2\ldots$, is an infinite sequence of states. For each $i\geq 0$, $q_{i+1}=\mu(s_i,v_i)$ for some $v_i \in \Sigma$. A run $\gamma$ is accepting if there exists a pair $(A,B) \in \Acc$ and $n\geq 0$ such that, for all $m\geq n$, we have $q_m \notin A$ and there exists infinitely many $k$ that satisfies $q_k \in B.$ Given an LTL specification $\varphi$ with atomic propositions $\Ap$, a DRA $\mathcal{A}_{\varphi}$ can be constructed with alphabet $2^{\mathcal{AP}}$ that accepts all words that satisfy the LTL specification $\varphi$~\cite{BK08}.

For an induced DTMC $\mathcal{M^{\sched}}$ of an MDP and a strategy $\sched$, a path $\varrho^{\sched}$ $=$ $s_0s_1\ldots$ generates a word $w$ $=$ $w_0w_1\ldots$ where $w_k$ $=$ $\mathcal{L}(s_k)$ for all $k \geq 0$. We denote the word that is generated by $\varrho^{\sched}$ as $\mathcal{L}(\varrho^{\sched})$. For an LTL specification $\varphi$, the set of words that is accepted by the DRA and satisfies the LTL specification $\varphi$ is given by $\{\varrho^{\sched}\in Path^{\sched}(\mathcal{M})\colon \mathcal{L}(\varrho^{\sched})\models\varphi\}$ , and is measurable \cite{BK08}. We define
\begin{align*}
\mathbb{P}_{\mathcal{M}^{\sched}}(\sinit \models \varphi)=\mathbb{P}_{\mathcal{M}^{\sched}}\lbrace \varrho^{\sched}\in Path^{\sched}(\mathcal{M}) \colon \mathcal{L}(\varrho^{\sched})\models \varphi\rbrace
\end{align*}
as the probability of satisfying the LTL specification $\varphi$ for an MDP $\mathcal{M}$ under the strategy $\sched \in \Sched^{\mathcal{M}}$.

The \emph{synthesis problem} is to find one particular strategy $\sched$ for an MDP $\mdp$ such that given an LTL specification $\varphi$ and a threshold $\beta\in[0,1]$, the induced DTMC $\mathcal{M^{\sched}}$ satisfies
\begin{align}\label{eq:speccons}
\mathbb{P}_{\mathcal{M}^{\sched}}(\sinit \models \varphi)\geq \beta,
\end{align}
which implies that the strategy $\sched$ satisfies the specification $\varphi$ with at least a probability of $\beta$.

We also consider expected cost properties $\varphi_c=\expRewProp{\kappa}{G}$, that restricts the expected cost to reach the set $G\subseteq S$ of goal states by an upper bound $\kappa\in\Q$. 
\begin{figure}[t]
	\centering
	\subfigure[MDP $\mdp$]
	{
		\scalebox{0.7}{\begin{tikzpicture}[scale=1, nodestyle/.style={draw,circle},baseline=(s0)]

    \node [nodestyle] (s0) at (0,0) {$s_0$};
    \node [nodestyle] (s1) [on grid, right=2.5cm of s0] {$s_1$};
    \node [nodestyle, accepting] (s2) [on grid, right=2.5cm of s1] {$s_2$};
    \node [nodestyle] (s3) [on grid, above=2.5cm of s1] {$s_3$};
    \node [nodestyle] (s4) [on grid, above=2.5cm of s2] {$s_4$};
    
    \node [circle, draw, scale=0.5, fill=black] (s0a) [right=.7cm of s0] {};
    \node [circle, draw, scale=0.5, fill=black] (s0b) [above=0.8cm of s0a] {};
    \node [circle, draw, scale=0.5, fill=black] (s1a) [right=.7cm of s1] {};
    \node [circle, draw, scale=0.5, fill=black] (s1b) [above=0.8cm of s1a] {};

    \draw [thick] (s0) -- node[above] {$a$} (s0a);
    \draw [thick] (s0) -- node[above] {$b$} (s0b);
    \draw [thick] (s1) -- node[above] {$c$} (s1a);
    \draw [thick] (s1) -- node[above] {$d$} (s1b);
    
    \draw[->] (s0a) -- node [auto] {\scriptsize$0.6$} (s1);
    \draw[->] (s0a) -- node [right, near end] {\scriptsize$0.4$} (s3);
    \draw[->] (s0b) -- node [auto] {\scriptsize$0.4$} (s1);
    \draw[->] (s0b) -- node [auto] {\scriptsize$0.6$} (s3);
    
    \draw[->] (s1a) -- node [auto] {\scriptsize$0.6$} (s2);
    \draw[->] (s1a) -- node [right, near end] {\scriptsize$0.4$} (s4);
    \draw[->] (s1b) -- node [auto] {\scriptsize$0.4$} (s2);
    \draw[->] (s1b) -- node [auto] {\scriptsize$0.6$} (s4);
    
    \draw(s4) edge[loop above] node [above] {\scriptsize$1$} (s4);
    \draw(s3) edge[loop above] node [above] {\scriptsize$1$} (s3);
    \draw(s2) edge[loop above] node [above] {\scriptsize$1$} (s2);    
    
   \end{tikzpicture}}\label{fig:mdp}
	}
	\subfigure[Induced MC $\mdp^{\sched_{unif}}$]
	{
		\scalebox{0.7}{\begin{tikzpicture}[scale=1, nodestyle/.style={draw,circle},baseline=(s0)]

    \node [nodestyle] (s0) at (0,0) {$s_0$};
    \node [nodestyle] (s1) [on grid, right=2cm of s0] {$s_1$};
    \node [nodestyle, accepting] (s2) [on grid, right=2cm of s1] {$s_2$};
    \node [nodestyle] (s3) [on grid, above=2.5cm of s1] {$s_3$};
    \node [nodestyle] (s4) [on grid, above=2.5cm of s2] {$s_4$};
    

    
    \draw[->] (s0) -- node [auto] {\scriptsize$0.5$} (s1);
    \draw[->] (s0) -- node [right, near end] {\scriptsize$0.5$} (s3);
    
    \draw[->] (s1) -- node [auto] {\scriptsize$0.5$} (s2);
    \draw[->] (s1) -- node [right, near end] {\scriptsize$0.5$} (s4);
    
    \draw(s4) edge[loop above] node [above] {\scriptsize$1$} (s4);
    \draw(s3) edge[loop above] node [above] {\scriptsize$1$} (s3);
    \draw(s2) edge[loop above] node [above] {\scriptsize$1$} (s2);    
    
   \end{tikzpicture}}\label{fig:induced_dtmc1}
	}
	\caption{MDP $\mdp$ with target state $s_2$ and induced MC for strategy $\sched_{\textit{unif}}$}
	\label{fig:mdp_example}
\end{figure}
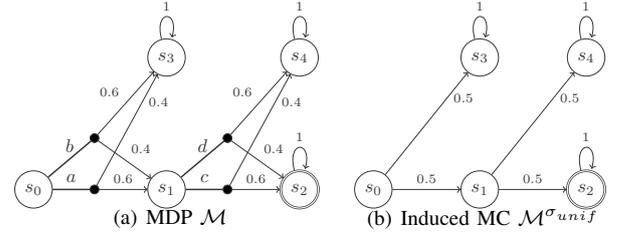
\begin{example}\label{ex:simple_mdp}
Fig.~\ref{fig:mdp} depicts an MDP $\mdp$ with initial state $s_0$. 
In state $s_0$, the available actions are $a$ and $b$. Similarly for state $s_1$, the two available actions are $c$ and $d$.
If action $a$ is selected in state $s_0$, the agent transitions to $s_1$ and $s_3$ with probabilities $0.6$ and $0.4$. 
For states $s_2,s_3$ and $s_4$ we omit actions, because of the self loops.

For a safety specification $\phi=\reachProp{0.21}{s_2}$, the deterministic strategy $\sched_{1}\in\Sched^\mdp$ with $\sched_{1}(s_0,a)=1$ and $\sched_{1}(s_1,c)=1$ induces a probability of $0.36$ to reach $s_2$.
Therefore, the specification is not satisfied, see the induced MC in Fig.~\ref{fig:induced_dtmc1}. 
Likewise, the randomized strategy $\sched_{\textit{unif}}\in\Sched^\mdp$ with $\sched_{\textit{unif}}(s_0,a)=\sched_{\textit{unif}}(s_0,b)=0.5$ and $\sched_{\textit{unif}}(s_1,c)=\sched_{\textit{unif}}(s_1,d)=0.5$ violates the specification, as the probability of reaching $s_2$ is $0.25$.
However, the deterministic strategy $\sched_{\textit{safe}}\in\Sched^\mdp$ with $\sched_{\textit{safe}}(s_0,b)=1$ and $\sched_{\textit{safe}}(s_1,d)=1$ induces a probability of $0.16$, thus $\sched_{\textit{safe}}\models\phi$. 
\end{example}

\subsection{Strategy synthesis in an MDP}\label{product_section}
Given an MDP, and an LTL specification $\varphi$, we aim to synthesize a strategy that satisfies $\varphi$, or equivalently, a strategy that satisfies the condition in~\eqref{eq:speccons}. 

\begin{definition} (Product MDP)
Let $\MdpInit$ be an MDP and $\DraInit$ be a DRA. 
The product MDP is a tuple $\ProductInit$ with a finite set $S_p=S \times Q$ of states, an initial state $\sinitp=(\sinit,q)\in S_p$ that satisfies $q=\delta(\qinit,\mathcal{L}(\sinit))$, a finite set $\Act$ of actions, a transition probability function $\probmdpp((s,q), a, (s',q'))= \probmdp(s,a,s')\;  \text{if} \quad q'=\delta(q,\mathcal{L}(s'))$, and  $\probmdpp((s,q), a, (s',q'))=0 \;  \text{if} \quad q'\neq\delta(q,\mathcal{L}(s'))$, a labeling function $\mathcal{L}_p((s,q))=\{q\}$, and the acceptance condition
$\Accp=\{(A_1^p,B_1^p),\ldots,(A_k^p,B_k^p) \}$ where $A_i^p=S\times A_i$ and $B_i^p=S\times B_i$ for all $(A_i,B_i)\in \Acc$ and $i=1,\ldots, k$.
\end{definition}

 \begin{definition}(AEC)
 The end component for the product MDP $\mdpp$ is given by a pair $(C,D)$, where a non-empty set $C\subseteq S_p$ of states and a function $D \colon C \rightarrow \Act$ is defined such that for any $s \in C$ we have
 $$\displaystyle \sum_{s' \in C}\probmdpp(s,D(s),s')=1,$$ 
 and the induced directed graph is strongly connected. An accepting end component (AEC) is an end component that satisfies $C \cap A^p_i=\emptyset$ and $C\cap B^p_i\neq \emptyset$ for some $i \in{1,\ldots,k}.$
\end{definition}

Given a product MDP $\mdpp$, we modify it to $\mdpp'$ by making all states in the end components absorbing, i.e., for all states $s \in C$, $\probmdpp(s,\act,s)=1$ for all $\act \in \Act$ in the modified MDP $\mdpp'$. 
Making all end components absorbing is commonly used in tools for model checking of LTL specifications in MDPs~\cite{BK08,KNP11,dehnert2017storm}.
We further assume that all states in the end component are absorbing.
The modification does not change the probability of satisfying an LTL specification as stated below. 

\begin{lemma}(From~\cite{de1997formal}) In each end component of an MDP, there exists a strategy in each state $s \in C$ that reaches any other state $s' \in C$ with a probability of 1.
\end{lemma}

A memoryless and randomized strategy for a product MDP $\mdpp$ is a function $\sched^p \colon S_p \rightarrow \Distr(\Act)$. A memoryless strategy $\sched^p$ is a finite-memory strategy $\sched'$ in the underlying MDP $\mdp$. Given a state $(s,q) \in S_p$, we consider $q$ to be a memory state and define $\sched'(\gamma)=\sched^p(s,q)$, where the run $\gamma=q_0 q_1 q_2\ldots q_n$ satisfies $q_n=q$ and $\Sigma(\qinit,L(\rho))=s.$ For the MDPs given in Definition 1 and LTL specifications, memoryless strategy in the product MDP $\mdpp$ are sufficient to achieve the maximal probability of satisfying the specification~\cite{BK08}.

Some states in the product MDP $\mdpp$ may be unreachable from the initial state $\sinitp$. These states do not affect the strategy synthesis in $\mdpp$, and can be removed from $\mdpp$. We assume that there is no unreachable states in the product MDP $\mdpp$. 
 
Let $\sched^p\in\Sched^{\mdpp}$ be a strategy for $\mdpp$ and let $\sched'\in\sched^{\mdp}$ be the strategy on $\mdp$ constructed from $\sched^p$ through the procedure explained above. 
The paths of the MDP $\mdp$ under the strategy $\sched'$ satisfy the LTL specification $\varphi$ with a probability of at least $\beta$, i.e., $\mathbb{P}_{\mathcal{M}^{\sched}}(\sinit \models \varphi)\geq \beta$, if and only if the paths of the induced DTMC $\mdpp^{\sched}$ from the product MDP $\mdpp$ under the strategy $\sched^p$ reach and stay in some AECs in $\mdpp$ with a probability of at least $\beta$~\cite{BK08}.

\section{Conceptual description of shared control}\label{sec:overview}
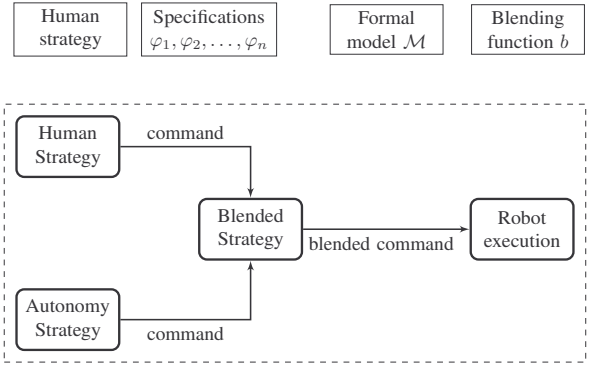
\begin{figure}[t]
\scalebox{0.73}{
	\centering
\begin{tikzpicture}
\tikzstyle{outer}= [draw, text centered, shape=rectangle, text width=1.8cm]
\tikzstyle{inner}=[draw, text centered, shape=rectangle, rounded corners, text width=1.5cm, minimum height=1.1cm, inner sep=5pt]			
\node[inner, very thick, text width=1.5cm] (human) {Human\\ Strategy};


\node[inner, below=2cm of human, very thick, text width=1.5cm] (autonomy) {Autonomy\\ Strategy};


\node[inner, right=1.4cm of human,yshift=-1.5cm, very thick, text width=1.5cm] (blending) {Blended\\ Strategy};



\node[inner, right=3cm of blending, very thick, text width=1.5cm] (execution) {Robot\\ execution};


\draw[thick,-latex'] (human.east) -| node[above, near start] {command} (blending);


\draw[thick,-latex'] (autonomy.east)  -|  node[below, near start] (autonomyBlending) {command} (blending);



\draw (blending) edge[-latex', thick] node[below, text width=2.8cm] {blended command} (execution);







%
%

%

\node (box) [draw, dashed, fit = (current bounding box), inner sep=0.2cm] {};

\node (input1) [outer,above=2.6cm of blending,xshift=5cm] {Blending function $\blendFunc$};
\node (input2) [outer,left=.5cm of input1] {Formal model $\mdp$};
\node (input3) [outer,left=3.5cm of input1, text width = 2.2 cm] {Specifications $\varphi_1,\varphi_2,\ldots,\varphi_n$};
\node (input4) [outer,left=6.2cm of input1] {Human\\ strategy};


%

\end{tikzpicture}
}
\caption{Shared control architecture.}
\label{fig:formal_setting}
\end{figure}
%
%
\noindent 
We now detail the general shared control concept adopted in this paper and state the formal problem.  
Consider the setting in Fig.~\ref{fig:formal_setting}.
As inputs, we have a set of task specifications, a model $\mdp$ for the robot behavior, and a blending function $b$.
The given robot task is described by certain performance and safety specifications $\varphi=\varphi_1\wedge\varphi_2\ldots\wedge\varphi_n$. 
For example, it may not be safe to take the shortest route because there may be too many obstacles in that route. 
In order to satisfy performance considerations, the robot should prefer to take the shortest route possible while not violating the safety specifications.
We model the behavior of the robot inside a stochastic environment as an MDP $\mdp$.


In our setting, a \emph{human} issues a set of commands for the robot to execute. 
It may be unrealistic that a human can grasp an MDP that models a realistic shared control scenario.
Indeed, a human will likely have difficulties interpreting a large number of possibilities and the associated probability of paths and payoffs~\cite{fryer2008categorical}, and it may be impractical for the human to provide the human strategy to the autonomy protocol, due to the possibly large state space of the MDP. 
Therefore, we compute a human strategy $\sched_h$ as an abstraction of a sequence of human's commands, which we obtain using inverse reinforcement learning~\cite{abbeel2004apprenticeship,ziebart2008maximum}.




We design an autonomy protocol that provides another strategy $\sched_a$, which we call the autonomy strategy.  
Then, we blend the two strategies according to the blending function $\blendFunc$ into the blended strategy $\sched_{ha}$. The blending function reflects preference over the human strategy or the autonomy strategy.
We ensure that the blended strategy deviates minimally from the human strategy. 

At runtime, we can then blend commands of the human with commands of the autonomy strategy. 
The resulting ``blended'' commands will induce the same behavior as the blended strategy $\sched_{ha}$, and the specifications are satisfied. 
Note that blending commands at runtime according to predefined blending function and autonomy protocol simply requires a linear combination of real values and is thus very efficient.

	The \emph{shared control synthesis problem} is then the synthesis of the repaired strategy $\sigma_{ha}$ such that it holds that $\sigma_{ha}\models\varphi$ while \emph{deviating minimally} from $\sigma_h$. 
	The deviation between the human strategy $\sched_h$ and the repaired strategy $\sched_{ha}$ is measured by the maximal difference between the two strategies in each state of the MDP. We state the problem that we study as follows.
	
	\begin{problem}
	Let $\mdp$ be an MDP, $\varphi$ be an LTL specification, $\sched_h$ be a human strategy, and $\beta$ be a constant. 
	Synthesize a repaired strategy $\sched_{ha} \in \Sched^{\mdp}$  that solves the following problem.
	\begin{align}
	\underset{\sched_{ha} \in \Sched^{\mdp}}{\textnormal{minimize}}&\quad \underset{s \in S, \act \in \Act}{\textnormal{max}} |\sched_h(s,\act) - \sched_{ha}(s,\act) |\\
	\textnormal{subject to}&\nonumber\\
	& \quad \mathbb{P}_{\mdp^{\sched_{ha}}}(\sinit \models \varphi)\geq \beta.
	\end{align}
	\end{problem}
For convenience, we will use the original MDP $\mdp$ instead of the product MDP $\mdpp$ in what follows as all concepts are directly transferrable.

\section{Synthesis of the autonomy protocol}\label{sec:synthesissec}

In this section, we describe our approach to synthesize an autonomy protocol for the shared control synthesis problem. We start by formalizing the concepts of strategy blending and strategy repair. We then show how we can synthesize a repaired strategy that deviates minimally from the human strategy based on quasiconvex programming. We  discuss how we can include additional specifications to the problem and discuss other measures for the human and the repaired strategy that induce a similar behavior.

\subsection{Strategy blending} \noindent Given the human strategy $\sigma_h\in\Sched^{\mdp}$ and the autonomy strategy $\sigma_a\in\Sched^{\mdp}$, a blending function computes a weighted composition of the two strategies by favoring one or the other strategy in each state of the MDP~\cite{javdani2015shared,dragan-et-al-assistive-teleoperation,dragan-et-al-policy-blending}.

Reference \cite{dragan-et-al-policy-blending} argues that the weight of blending shows the \emph{confidence} in how well the autonomy protocol can assist to perform the human's task.
Put differently, the blending function should assign a \emph{low confidence} to the actions of the human if they may lead to a violation of the specifications.
Recall Fig.~\ref{fig:gridworld} and the example in the introduction.
In the cells of the gridworld where some actions may result in a collusion with the vacuum cleaner with a high probability, it makes sense to assign a higher weight to the autonomy strategy.    

We pick the blending function as a \emph{state-dependent} function that weighs the confidence in both the human strategy and the autonomy strategy at each state of the MDP $\mdp$~\cite{javdani2015shared,dragan-et-al-assistive-teleoperation,dragan-et-al-policy-blending}.


\begin{definition}[Linear blending]\label{def:blending}
Given the  MDP $\MdpInit$, two memoryless strategies $\sched_h,\sched_a\in\Sched^{\mdp}$, and a \emph{blending function} $\blendFunc\colon S\rightarrow [0,1]$,  \emph{the blended strategy $\sigma_{ha}\in\Sched^{\mdp}$} for all states $s\in S$, and actions $\act\in\Act$ is
			\begin{align}\label{eq:blending}
				\sigma_{ha}(s,\alpha)=\blendFunc(s)\cdot\sched_h(s,\alpha) + (1-\blendFunc(s))\cdot\sched_a(s,\alpha).
			\end{align}			
\end{definition}
For each $s\in S$, the value of $\blendFunc(s)$ represents the ``weight'' of $\sigma_h$ at $s$, meaning how much emphasis the blending function puts on the human strategy at state $s$. For example, referring back to Fig.~\ref{fig:gridworld}, the critical cells of the gridworld correspond to certain states of the MDP $\mdp$. At these states, we may assign a very low confidence in the human strategy. For instance at such a state $s\in S$, we might have $b(s)=0.1$, meaning the blended strategy in state $s$ puts more emphasis on the autonomy strategy $\sigma_a$.

\subsection{Solution to the shared control synthesis problem}\label{sec:synthesislp}
\noindent In this section, we propose an algorithm for solving the shared control synthesis problem.
Our solution is based on quasiconvex programming which can be solved by checking feasibility of a number of convex optimization problems.
We show that the result of the quasiconvex program is the repaired strategy as in Problem~1. 
The strategy satisfies the task specifications while deviating minimally from the human strategy.
We use that result to compute the autonomy strategy $\sched_a$ that may then be subject to blending.


\subsubsection{Perturbation of strategies}
\noindent As mentioned in the introduction, the blended strategy should deviate minimally from the human strategy. To measure the quantity of such a deviation, we introduce the concept of \emph{perturbation}, which was used in~\cite{chen-et-al-concur-2014-pertubation}. 
	To modify a (randomized) strategy, we employ \emph{additive perturbation} by increasing or decreasing the probabilities of action choices in each state. We also ensure that for each state, the resulting strategy is a well-defined distribution over the actions.

\begin{definition}[Strategy perturbation]\label{def:perturbationlp}
	Given the  MDP $\mdp$ and a strategy $\sched\in\Sched^{\mdp}$, a \emph{perturbation} $\delta$ is a function $\delta\colon S\times\Act\rightarrow[-1,1]$ with
	\begin{align*}
		\sum_{\act\in\Act}\delta(s,\act)=0 \quad \forall s\in S. 
	\end{align*}	
	The \emph{perturbation value} at state $s$ for action $\act$ is $\delta(s,\act)$. Overloading the notation, the \emph{perturbed strategy} $\delta(\sched)$ is 
	\begin{align}
		\delta(\sched)(s,\act)=\sched(s,\act)+\delta(s,\act) \quad \forall s\in S,\act\in\Act.\label{eq:perturbationlp} 
	\end{align}
\end{definition}

\subsubsection{Dual linear programming formulation for MDPs}

In this section, we recall the LP formulation to compute a strategy that maximizes the probability of satisfying a specification $\varphi$ in an MDP~\cite{puterman2014markov,forejt2011quantitative}. 
Let $B$ the set of states in accepting end components in $\mdp$ (or in fact within in the product MDP $\mdpp$) and let $S_r$ be the set of all states that are not in $B$ and have nonzero probability of reaching a state $s \in B$. These sets can be computed in time polynomial in the size of $\mdp$ by doing a graph search over the  MDP $\mdp$~\cite{BK08}. In this section, we assume that there exists a strategy $\sched \in \Sched^{\mdp}$ that satisfies an LTL formula with a probability of at least $\beta$, which can be verified in time polynomial by solving a linear programming problem~\cite{BK08}.

The variables of the dual LP formulation are following:
		%
		\begin{itemize}
			\item $x_{\sched_{ha}}(s,\alpha)\in [0,\infty)$ for each state $s\in S_r$ and action $\act\in\Act$ defines the occupancy measure of a state-action pair for the strategy $\sched_{ha}$, i.e., the expected number of times of taking action $\act$ in state $s$. 
			\item $x_{\sched_{ha}}(s) \in [0, 1]$ for each state $s \in B$ defines the probability of reaching a state $s\in B$ in an accepting end component.
		\end{itemize}

\begin{align}
		\displaystyle	&\text{maximize} \quad  \sum_{s\in B}x_{\sched_{ha}}(s)\label{eq:strategylp:obj}\\
			&\text{subject to} \nonumber
\\
		&\displaystyle		\forall s\in S_r.	\nonumber\\
		\quad & \sum_{\act\in\Act}x_{\sched_{ha}}(s,\act) = \sum_{s'\in S_r}\sum_{\act\in\Act}\probmdp(s',\act,s)x_{\sched_{ha}}(s',\act)+\alpha_s\label{eq:strategylp:welldefined_sched}\\
				&\displaystyle		\forall s\in B.	\nonumber\\
	\displaystyle&	x_{\sched_{ha}}(s)= \sum_{s'\in S_r}\sum_{\act\in\Act}\probmdp(s',\act,s)x_{\sched_{ha}}(s',\act)+\alpha_s\label{eq:strategylp:mec_sched}\\
		\displaystyle		&	 \sum_{s \in B} x_{\sched_{ha}}(s) \geq \beta\label{eq:strategylp:probthresh}
\end{align}
where $\alpha_s=1$ if $s=\sinit$ and $\alpha_s=0$ if $s\neq \sinit$. The constraints~\eqref{eq:strategylp:welldefined_sched} and~\eqref{eq:strategylp:mec_sched} ensure that the expected number of times transitioning to a state $s \in S$ is equal to the expected number of times to take action $\act$ that transitions to a different state $s' \in S$. The constraint~\eqref{eq:strategylp:probthresh} ensures that the specification $\varphi$ is satisfied with a probability of at least $\beta$. We refer the reader to~\cite{puterman2014markov,forejt2011quantitative} for details about the constraints in the LP.

For any optimal solution $x$ to the LP in~\eqref{eq:strategylp:obj}--\eqref{eq:strategylp:probthresh}, 
\begin{align}
\displaystyle \sched_{ha}(s,\act)= \dfrac{x_{\sched_{ha}}(s,\act)}{ \displaystyle\sum_{\act\in\Act}x_{\sched_{ha}}(s,\act)}\label{eq:occupmeasure}
\end{align} 
is an optimal strategy, and $x_{\sched_{ha}}$ is the occupancy measure of $\sched_{ha}$, see~\cite{puterman2014markov} and~\cite{forejt2011quantitative} for details. After finding an optimal solution to the LP in~\eqref{eq:strategylp:obj}--\eqref{eq:strategylp:probthresh}, we can compute the probability of satisfying a specification by $$\sum_{s \in B} x_{\sched_{ha}}(s).$$

%

\subsubsection{Strategy repair using quasiconvex programming}

\noindent Given the human strategy, $\sched_h\in\Sched^{\mdp}$, the aim of the autonomy protocol is to compute the blended strategy, or the repaired strategy $\sched_{ha}$ that induces a similar behavior to the human strategy while satisfying the specifications. We compute the repaired strategy by perturbing the human strategy, which is introduced in Definition~\ref{def:perturbationlp}. We show our formulation to compute the repaired strategy in the following Lemma.

\begin{lemma}
The shared control synthesis problem can be formulated as the following nonlinear programming program with following variables:

		\begin{itemize}
			\item \emph{$x_{\sched_{ha}}(s,\alpha)\in [0,\infty)$ for each state $s\in S_r$ and action $\act\in\Act$ and $x_{\sched_{ha}}(s) \in [0, 1]$ for each state $s \in B$ as defined for the optimization problem in~\eqref{eq:strategylp:obj}--\eqref{eq:strategylp:probthresh}.}
			\item \emph{$\hat{\delta} \in [0, 1]$ gives the maximal deviation between the human strategy $\sched_h$ and the repaired strategy $\sched_{ha}$.}
		\end{itemize}

\begin{align}
		\displaystyle	&\emph{\text{minimize}} \quad  \hat{\delta}\label{eq:l1_strategylp:obj}\\
			&\emph{\text{subject to}} \nonumber
\\
		&\displaystyle		\forall s\in S_r.	\nonumber\\
		\quad & \sum_{\act\in\Act}x_{\sched_{ha}}(s,\act) = \sum_{s'\in S_r}\sum_{\act\in\Act}\probmdp(s',\act,s)x_{\sched_{ha}}(s',\act)+\alpha_s\label{eq:l1_strategylp:welldefined_sched}\\
				&\displaystyle		\forall s\in B.	\nonumber\\
	\displaystyle&	x_{\sched_{ha}}(s)= \sum_{s'\in S_r}\sum_{\act\in\Act}\probmdp(s',\act,s)x_{\sched_{ha}}(s',\act)+\alpha_s\label{eq:l1_strategylp:mec_sched}\\
		\displaystyle		&	 \sum_{s \in B} x_{\sched_{ha}}(s) \geq \beta\label{eq:l1_strategylp:probthresh}\\
		&\displaystyle		\forall s\in S_r.	\nonumber\\
		&	\vert x_{ha}(s,\act)-\sum_{\act\in\Act}x_{ha}(s,\act)\sigma_h(s,\act)\vert\leq\hat{\delta}\sum_{\act\in\Act}x_{ha}(s,\act).\label{eq:l1_perturbationlpsynthesis_occup_quasi2}
\end{align}

\end{lemma}

\begin{proof}
For any solution to the optimization problem above, the constraints in~\eqref{eq:l1_strategylp:welldefined_sched}--\eqref{eq:l1_strategylp:probthresh} ensure that the strategy computed by~\eqref{eq:occupmeasure} satisfies the specification. We now show that by minimizing $\hat{\delta}$, we minimize the maximal deviation between the human strategy and the repaired strategy.

As in Definition~\ref{def:perturbationlp}, we perturb the human strategy $\sched_h$ to the repaired strategy $\sched_{ha}$ by

	\begin{align}
	\forall s\in S_r.\act\in\Act. \quad	\sched_{ha}(s,\act)=\sched_h(s,\act)+\delta(s,\act). \label{eq:perturbationlpsynthesis}
	\end{align}
		
Note that this constraint is not a function of the occupancy measure of $\sched_{ha}$. By using the definition of occupancy measure in~\eqref{eq:occupmeasure}, we reformulate the constraint in~\eqref{eq:perturbationlpsynthesis} into the constraint

	\begin{align}
  \displaystyle\forall s\in S_r.\act\in\Act.	\quad	\dfrac{x_{ha}(s,\act)}{ \displaystyle\sum_{\act\in\Act}x_{ha}(s,\act)}=\sched_h(s,\act)+\delta(s,\act) \label{eq:perturbationlpsynthesis_occup}
	\end{align}
or equivalently to the constraint
		\begin{align}
	& \forall s\in S_r.\act\in\Act.\nonumber\\
	& \quad	x_{ha}(s,\act)=\sum_{\act\in\Act}x_{ha}(s,\act) \left(\sched_h(s,\act)+\delta(s,\act)\right). \label{eq:perturbationlpsynthesis_occup_lp}
	\end{align}
	 
Since we are interested in minimizing the maximal deviation, we assign a common variable $\hat{\delta} \in [0,1]$ for all state-action pairs in the MDP $\mdp$ to put an upper bound on the deviation by
		\begin{align}
	&\forall s\in S_r.\act\in\Act. \nonumber \\
	&\vert x_{ha}(s,\act)-\sum_{\act\in\Act}x_{ha}(s,\act)\sigma_h(s,\act)\vert\leq&\hat{\delta}\sum_{\act\in\Act}x_{ha}(s,\act). \label{eq:perturbationlpsynthesis_occup_quasi}
	\end{align}
	
Therefore, by minimizing $\hat{\delta}$ subject to the constraints in~\eqref{eq:l1_strategylp:welldefined_sched}--\eqref{eq:l1_perturbationlpsynthesis_occup_quasi2} we ensure that the repaired strategy $\sched_{ha}$ deviates minimally from the human strategy $\sched_{h}$.
	\end{proof}
The constraint in~\eqref{eq:perturbationlpsynthesis_occup_quasi} is a nonlinear constraint. In fact, it is a quadratic constraint due to multiplication of $\hat{\delta}$ and $x_{ha}$. However, we show that the problem in~\eqref{eq:l1_strategylp:obj}--\eqref{eq:l1_perturbationlpsynthesis_occup_quasi2} is a quasiconvex programming problem, which can be solved efficiently using bisection over $\hat{\delta}$~\cite{boyd_convex_optimization}.


\begin{lemma}
The constraint in~\eqref{eq:perturbationlpsynthesis_occup_quasi} is quasiconvex, therefore the nonlinear programming problem in~\eqref{eq:l1_strategylp:obj}--\eqref{eq:l1_perturbationlpsynthesis_occup_quasi2} is a quasiconvex programming problem.
\end{lemma}

\begin{proof}
For a fixed $\hat{\delta}$, the set described by the inequality in~\eqref{eq:perturbationlpsynthesis_occup_quasi}  is convex, that is, the \emph{sublevel sets} of the function are convex~\cite[Section 3.4]{boyd_convex_optimization}. Therefore, the constraint in~\eqref{eq:perturbationlpsynthesis_occup_quasi} is quasiconvex and the nonlinear programming problem in~\eqref{eq:l1_strategylp:obj}--\eqref{eq:l1_perturbationlpsynthesis_occup_quasi2} is a quasiconvex programming problem (QCP).
\end{proof}

We solve the QCP in~\eqref{eq:l1_strategylp:obj}--\eqref{eq:l1_perturbationlpsynthesis_occup_quasi2} by employing bisection over the variable $\hat{\delta}$. We initialize a lower and upper bound of the maximal deviation between the human strategy and the repaired strategy to $0$ and $1$ respectively. Then, we iteratively refine the bounds by solving a number of convex feasibility problems. A method to solve quasiconvex optimization problems is given in~\cite[Algorithm 4.1]{boyd_convex_optimization}. Our approach is given in Algorithm~1 based on the Algorithm 4.1 in~\cite{boyd_convex_optimization}. We now state the main result of the paper.

\begin{algorithm}
\label{al:algo}
\DontPrintSemicolon
\textbf{given} $\MdpInit$, $\sched_h$, $l=0, u=1$, tolerance $\epsilon >0$.\;
\textbf{repeat}\;
\quad 1. Set $\hat{\delta}=(l+u)/2.$\;
\quad 2. Solve the convex feasibility problem in~\eqref{eq:l1_strategylp:welldefined_sched}--\eqref{eq:l1_perturbationlpsynthesis_occup_quasi2}.\;
\quad 3. \textbf{if} the problem in~\eqref{eq:l1_strategylp:welldefined_sched}--\eqref{eq:l1_perturbationlpsynthesis_occup_quasi2} is feasible, \textbf{then}\; \qquad\qquad $u:= \hat{\delta}$,
$\displaystyle \sched_{ha}(s,\act)= \dfrac{x_{\sched_{ha}}(s,\act)}{ \displaystyle\sum_{\act\in\Act}x_{\sched_{ha}}(s,\act)}$\;\;
 \qquad \enskip\textbf{else}  $l:= \hat{\delta}$.\;
\textbf{until} $u-l\leq \epsilon$.\;
\caption{Bisection method to synthesize an optimal repaired strategy $\sched_{ha}$ for the shared control synthesis problem.\label{IR}}
\end{algorithm}
	


\begin{theorem}\label{theo:correctness}
	The repaired strategy $\sched_{ha}$ obtained from Algorithm 1 satisfies the task specifications and it deviates minimally from the human strategy $\sched_{h}$, and is an optimal solution to the shared control synthesis problem.
\end{theorem}

\begin{proof}
From a satisfying assignment to the constraints in~\eqref{eq:l1_strategylp:obj}--\eqref{eq:l1_perturbationlpsynthesis_occup_quasi2}, we compute a strategy that satisfies the specification using~\eqref{eq:occupmeasure}. Using Algorithm~\ref{al:algo}, we can compute the repaired strategy $\sigma_{ha}$ that deviates minimally from the human strategy $\sigma_{h}$ up to $\epsilon$ accuracy in $\ceil*{\log_2(\dfrac{1}{\epsilon})}$ iterations. Therefore, Algorithm~\ref{al:algo} computes an optimal strategy for the shared control synthesis problem.
\end{proof}

The strategy given by Algorithm~\ref{al:algo} computes the minimally deviating repaired strategy $\sched_{ha}$ that satisfies the LTL specification. 
In~\cite{jansen2017synthesis}, we considered computing an autonomy protocol with a greedy approach. That approach requires solving possibly an unbounded number of LPs to compute a feasible strategy that is not necessarily optimal. On the other hand, using Algorithm~1, we only need to check feasibility of a number of LPs that can be determined to compute an optimal strategy. Note that we do not compute the autonomy strategy $\sched_{a}$ with the QCP in~\eqref{eq:l1_strategylp:obj}--\eqref{eq:l1_perturbationlpsynthesis_occup_quasi2} directly. After computing the repaired strategy $\sigma_{ha}$, we compute the autonomy strategy $\sigma_a$ according to the Definition~\ref{def:blending}.

Computationally, the most expensive step of the Algorithm~\ref{al:algo} is checking the feasibility of the optimization problem in~\eqref{eq:l1_strategylp:welldefined_sched}--\eqref{eq:l1_perturbationlpsynthesis_occup_quasi2}. The number of variables and constraints in the optimization problem are linear in the number of states and actions in $\mdp$, therefore, checking feasibility of the optimization problem can be done in time polynomial in the size of $\mdp$ with interior point methods~\cite{nesterov1994interior}. Algorithm 1 terminates after $\ceil*{\log_2(\dfrac{1}{\epsilon})}$ iterations, therefore we can compute an optimal strategy up to $\epsilon$ accuracy in time polynomial in the size of $\mdp$.

\subsubsection{Additional specifications}
\noindent The QCP in~\eqref{eq:l1_strategylp:obj}--\eqref{eq:l1_perturbationlpsynthesis_occup_quasi2} computes an optimal strategy for a single LTL specification $\varphi$. Suppose that we are given a reachability specification $\varphi_r=\reachProp{\lambda}{T}$ with $T\in S$ in addition to the LTL specification $\varphi$. We can handle this specification by appending the constraint	
\begin{align}
		\displaystyle		&	\quad \sum_{s\in B} x_{\sched_{ha}}(s) \geq \lambda	 \label{eq:strategylp:reachprop}
		\end{align}		
to the QCP in~\eqref{eq:l1_strategylp:obj}--\eqref{eq:l1_perturbationlpsynthesis_occup_quasi2}. The constraint in~\eqref{eq:strategylp:reachprop} ensures that the probability of reaching $T$ is greater than $\lambda$.

We handle an \emph{expected cost specification} $\expRewProp{\kappa}{G}$ for $G\subseteq S$, by adding the constraint
	\begin{align}
		\displaystyle		&	\quad \sum_{s\in S_r \setminus G}\sum_{\act\in\Act} C(s,\act)x_{\sched_{ha}}(s,\act) \leq \kappa \label{eq:strategylp:expprop}
		\end{align}
to the QCP in~\eqref{eq:l1_strategylp:obj}--\eqref{eq:l1_perturbationlpsynthesis_occup_quasi2}. The constraint in~\eqref{eq:strategylp:reachprop} ensures that the expected cost of reaching $G$ is less than $\kappa$.

\subsubsection{Additional measures}
\noindent 
We discuss additional measures that can be used to render the behavior between the human and the autonomy protocol similar based on the occupancy measure of a strategy.
Instead of minimizing the maximal deviation between the human strategy and the repaired strategy, we can also minimize the maximal difference of occupancy measures of the strategies. In this case, the difference between the human strategy and the repaired strategy will be smaller in states where the expected number of being in a state is higher, and will be higher if the state is not visited frequently. We can minimize the maximal difference of occupancy measures by adding the following objective to the constraints in~\eqref{eq:l1_strategylp:welldefined_sched}--\eqref{eq:l1_strategylp:probthresh}:
\begin{align*}
\displaystyle	&\emph{\textnormal{minimize}} \quad \underset{s \in S, \act \in \Act}{\textnormal{max}} |x_{\sched_{ha}}(s,\act) - x_{\sched_{h}}(s,\act) |
\end{align*}
or, equivalently,
\begin{align*}
\displaystyle	&\emph{\textnormal{minimize}} \quad ||x_{\sched_{ha}} - x_{\sched_{h}} ||_{\infty}.
\end{align*}
The occupancy measure of the human strategy can be computed by finding a feasible solution to the constraints in~\eqref{eq:l1_strategylp:welldefined_sched}--\eqref{eq:l1_strategylp:mec_sched} for the induced DTMC $\mdp^{\sched_{ha}}$. We can also minimize further convex norms of the human strategy and the repaired strategy, such as 1-norm or 2-norm.
%

%
%
\section{Case study and experiments}\label{sec:simulation}
\noindent
We present two numerical examples that illustrate the efficacy of the proposed approach. In the first example, we consider a wheelchair scenario from Fig.~\ref{fig:gridworld}. The goal in this scenario is to reach the target state while not crashing with the obstacle. In the second example, we consider an unmanned aerial vehicle (UAV) mission, where the objective is to survey certain regions while avoiding enemy agents.

We require an abstract representation of the human's commands as a strategy to use our synthesis approach in a shared control scenario.
We first discuss how such strategies may be obtained using inverse reinforcement learning and report on case study results.

\begin{figure}[t]
	\centering\scalebox{0.7}{{\includegraphics[scale=0.24]{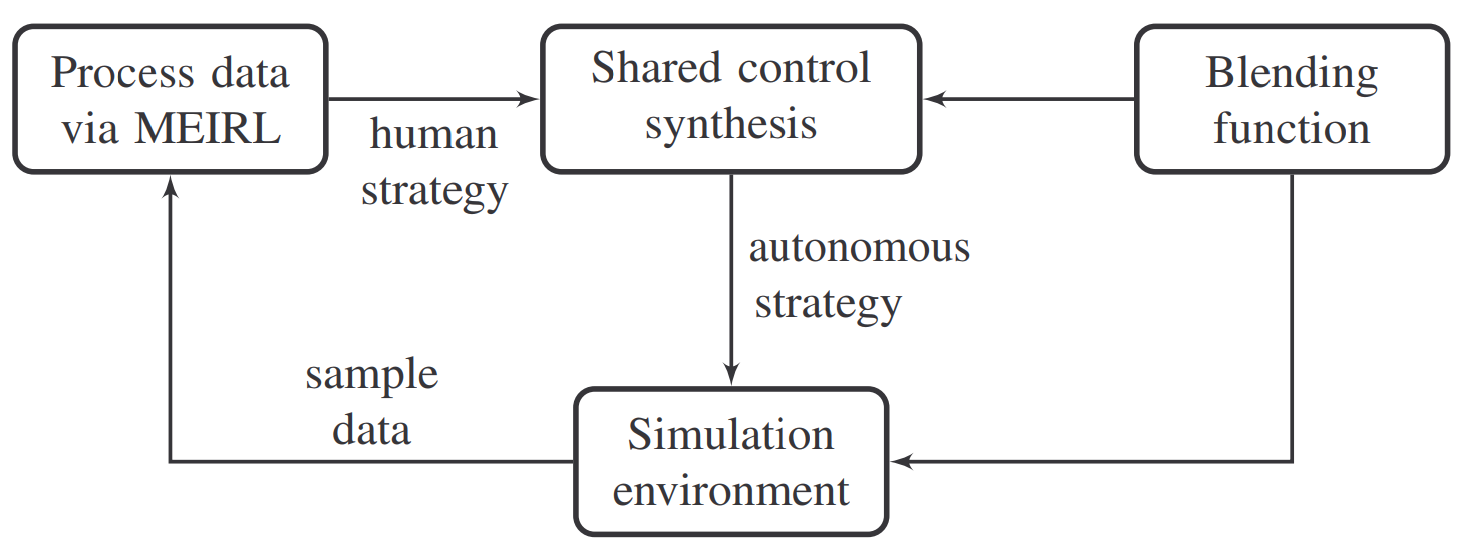}}}
	\caption{The setting of the case study for the shared control simulation. We collect sample data from a simulation environment, and compute the human strategy using maximum-entropy inverse reinforcement learning (MEIRL). From the human strategy, we compute an autonomous strategy based on our approach to the shared control synthesis problem. }	
	\label{fig:data}
\end{figure}

\subsection{Experimental setting}
We give an overview of the workflow of the experiments in Fig.~\ref{fig:data}. In an simulation environment, we collect sample data from the human's commands. Based on these commands, we compute a human strategy $\sched_h$ using maximum-entropy inverse reinforcement learning (MEIRL)~\cite{ziebart2008maximum}. After computing the human strategy, we synthesize the repaired strategy $\sched_{ha}$ using the procedure in Algorithm 1. After synthesizing the repaired strategy, we compute the autonomous strategy using~\eqref{eq:blending}. We can further refine our representation of the human strategy by collecting more sample data from the human's commands before blending with the autonomous strategy.

We model the wheelchair scenario inside an interactive \tool{Python} environment. In the second scenario, we use the UAV simulation environment AMASE\footnote{https://github.com/afrl-rq/OpenAMASE}, developed at Air Force Research Laboratory. AMASE can be used to simulate multi-UAV missions. 
The graphical user interfaces of AMASE allow humans to send commands to one or multiple vehicles at run time. It includes three main programs: a simulator, a data playback tool, and a scenario setup tool. 

We use the model checker \tool{PRISM}~\cite{KNP11} to verify if the computed strategies satisfy the specification. 
We use the LP solver \tool{Gurobi}~\cite{gurobi} to check the feasibility of the LP problems that is given in Section~\ref{sec:synthesissec}. We also implemented the greedy approach for strategy repair in~\cite{jansen2017synthesis}. In this section, we refer to the procedure given by Algorithm~1 as QCP method, and the procedure from~\cite{jansen2017synthesis} as greedy method.

\subsection{Data collection}
We asked five participants to accomplish tasks in the wheelchair scenario. The goal is moving the wheelchair to a target cell in the gridworld while never occupying the same cell as the moving obstacle. Similarly, three participants performed the surveillance task in the AMASE environment. 

From the data obtained from each participant, we compute an individual randomized human strategy $\sched_h$ via MEIRL. Reference~\cite{javdani2015shared} uses inverse reinforcement learning to reason about the human's commands in a shared control scenario from human's demonstrations. However, they lack formal guarantees on the robot's execution. 

 In our setting, we denote each \emph{sample} as one particular command of the participant, and we assume that the participant issues the command to satisfy the specification. Under this assumption, we can bound the probability of a possible deviation from the actual intent with respect to the number of samples using Hoeffding's inequality for the resulting strategy, see~\cite{ziebart2010modeling} for details. Using these bounds, we can determine the required number of commands to get an approximation of a typical human behavior. 
The probability of a possible deviation from the human behavior is given by $\mathcal{O}(\exp(-n\gamma^2))$, where $n$ is the number of commands from the human and $\gamma$ is the upper bound on the deviation between the probability of satisfying the specification with the true human strategy and the probability obtained by the strategy that is computed by inverse reinforcement learning.
For example, to ensure an upper bound $\gamma = 0.05$ on the deviation of the probability of satisfying the specification with a probability of $0.99$, we require $1060$ demonstrations from the human. 


We design the blending function by assigning a low weight to the human strategy at states where it yields a low probability of reaching the target set. 
Using this function, we create the autonomy strategy $\sched_a$ and pass it (together with the blending function) back to the environment. Note that the repaired strategy $\sched_{ha}$ satisfies the specification, by Theorem~\ref{theo:correctness}. 

\subsection{Gridworld}
The size of the gridworld in  Fig.~\ref{fig:gridworld} is variable, and we generate a number of randomly moving (e.g., the vacuum cleaner) and stationary obstacles. An agent (e.g., the wheelchair) moves in the gridworld according to the commands from a human. For the gridworld scenario, we construct an MDP where the states represent the positions of the agent and the obstacles and the actions induce changes in the agent position.

The safety specification states that the agent has to reach a target cell while not crashing into an obstacle with a certain probability $\beta\in[0,1]$, formally $\pctlProb_{\geq \beta}(\neg \texttt{crash } \pctlUntil \texttt{ target})$. 

First, we report results for one particular participant in a gridworld scenario with a $8\times 8$ grid and one moving obstacle. 
The states of the MDP are generated by the Cartesian product of the states of the agent and the obstacle. 
The agent and the obstacle have four actions in all states, namely left, right, up and down. At each state, a transition to the chosen direction occurs with a probability of 0.7, and the agent transitions to each adjacent state in the chosen direction with a probability 0.15. 
If a transition to the wall occurs, the agent remains in the same state. We fix a particular strategy for the obstacle, and determine the transition probabilities between states as a product of transitioning to the next states for the agent and the obstacle. The resulting MDP has $2304$ states and $36864$ transitions. 
We compute the human strategy using MEIRL where the features are the components of the cost function of the human, for instance the distance to the obstacle and the goal state. 

We instantiate the safety specification with $\beta=0.7$, which means the target should be reached with at least a probability of $0.7$. 
The human strategy $\sched_{h}$ induces a probability of $0.546$ to satisfy the specification. That is, it does not satisfy the specification. 

We compute the repaired strategy $\sigma_{ha}$ using the greedy and the QCP approach, and both strategies satisfy the specification with a probability larger than $\beta$.
On the one hand, the maximum deviation between $\sched_{h}$ and $\sched_{ha}$ is 0.15 with the greedy approach, which implies that the strategy of the human and the autonomy protocol deviates at most 15\% for all states and actions. 
On the other hand, the maximum deviation between $\sched_{h}$ and $\sched_{ha}$ is 0.03 with the QCP approach. The results show that the QCP approach computes a repaired strategy that induces a more similar strategy to the human strategy compared to the LP approach.

\begin{figure}[t]
	\subfigure[Strategy $\sched_h$]{
		\includegraphics[scale=0.4]{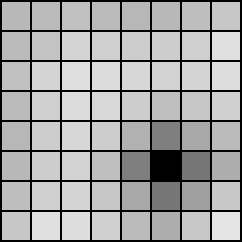}
	}
		\subfigure[Strategy  $\sched_{ah}$]{
		\includegraphics[scale=0.4]{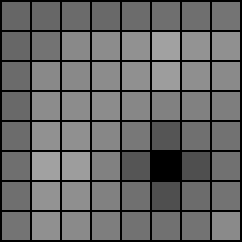}
	}
	\subfigure[Strategy $\sched_a$]{
		\includegraphics[scale=0.4]{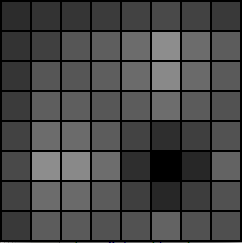}
	}
	\caption{Graphical representation of the obtained human, blended, and autonomy strategy in the grid.}
	\label{fig:heatmap}
\end{figure}
%

We give a graphical representation of the human strategy $\sched_h$, repaired strategy $\sched_{ha}$, and the autonomy strategy $\sched_{a}$ in Fig.~\ref{fig:heatmap}. For each strategy, we indicate the average probability of safely reaching the target with the QCP approach. Note that the probability of reaching the target depends on the current position of the obstacle. Therefore, the probability for satisfying a specification could be higher or lower than shown in Fig.~\ref{fig:heatmap}.
In Fig.~\ref{fig:heatmap}, the probability of reaching the target increases with a darker color, and black indicates a probability of $1$ to reach the target.  
We observe that the human strategy induces a lower probability of reaching the target in most of the states, while for the repaired strategy, the probability of reaching target is higher in all cells. 
Note that the autonomy strategy induces a very high probability of reaching the target in each cell, but the autonomy strategy may be too safe and may not be similar to the human strategy.
%
%
%

%
%
%
%

\begin{table*}[t]
\centering
\caption{Scalability results for the gridworld example. We list the synthesis time of the both approaches in seconds. '$\delta_{\textnormal{G}}$' and $\delta_{\textnormal{QCP}}$' refer to the maximal deviation of the greedy and QCP approach.}
\label{tab:scalability1}
\scalebox{1}{\begin{tabular}{@{}cccccccc@{}}
\toprule
Gridworld size         & Number of states     & Number of transitions     & \vtop{\hbox{\strut Synthesis time with the}\hbox{\thinspace\thinspace greedy approach (sec)}} & $\delta_{\textnormal{G}}$  & \vtop{\hbox{\strut Synthesis time with the}\hbox{\strut \enskip \thinspace QCP approach (sec)}}  & $\delta_{\textnormal{QCP}}$ \\
\midrule
$8\times 8$    &  $2,304$    & $36,864$  & $14.12$ & $0.145$  & $31.49$ & $0.031$ \\
$10\times 10$   &  $3,600$    & $57,600$  & $23.80$ & $0.231$ &  $44.61$ & $0.042$		\\ 
$12\times 12$    &  $14,400$   & $230,400$   & $250.78$ & $0.339$ & $452.27$ & $0.050$                   \\        
$20\times 20$    &  $40,000$   & $640,000$   & $913.23$ & $0.373$ & $1682.05$ & $0.048$                   \\        
\bottomrule\end{tabular}}
\end{table*} 

To finally assess the scalability of our approach, consider Table~\ref{tab:scalability1}. We generated MDPs for different gridworlds with a different number of states and number of obstacles. We list the number of states in the MDP and the number of transitions. We report on the time that the synthesis process took with the greedy approach and QCP approach, which includes the time of solving the LPs in the greedy method or QCPs measured in seconds. It also includes the model checking times using PRISM for the greedy approach. To represent the optimality of the synthesis, we list the maximal deviation between the repaired strategy and the human strategy for the greedy and QCP approach (labeled as "$\delta_{\textnormal{G}}$" and "$\delta_{\textnormal{QCP}}$"). In all of the examples, we observe that the strategies obtained by the QCP approach yield autonomy strategies with less deviation to the human strategy while having similar computation time with the greedy approach.
 
\begin{figure}[t]\centering
	\subfigure[Snapshot of a simulation using the AMASE simulator. The objective of the agent is to keep surveilling the green regions while avoiding enemy agents and restricted operating zones.]{{\includegraphics[scale=0.251]{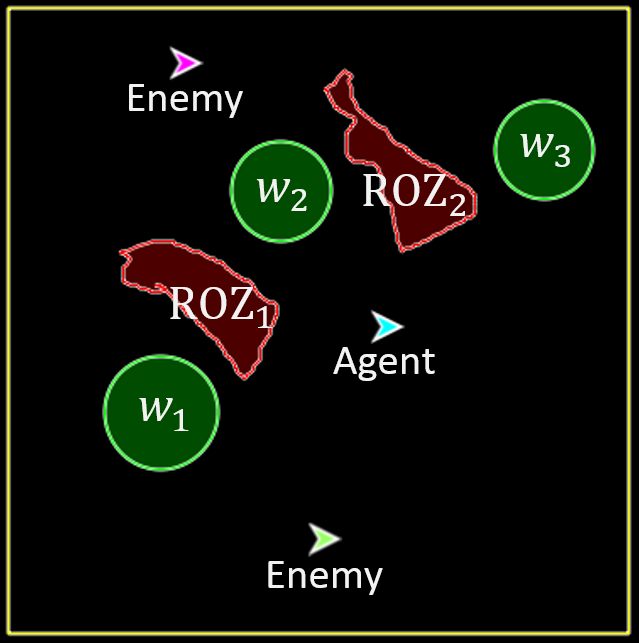}}}
	\subfigure[The graphical user interface of the AMASE simulator for a UAV mission. The user interface contains various information about the vehicles such as the speed and the heading.]{{\includegraphics[scale=0.305]{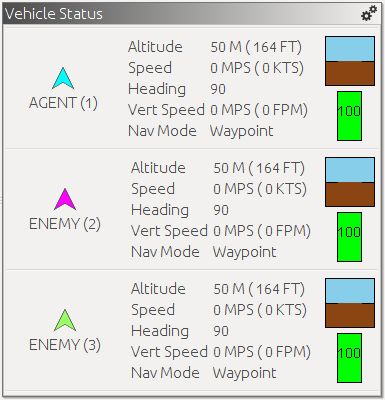}}}
	\caption{An example of UAV mission that is simulated on AMASE.}
	\label{fig:amase}
\end{figure}

\subsection{UAV mission planning}
Similar to the gridworld scenario, we generate an MDP where states denote the position of the agent and the enemy agents in an AMASE scenario. Consider an example scenario in Fig.~\ref{fig:amase}: The specification (or the mission) of the agent (blue UAV) is to keep surveilling the green regions (labeled as $w_1, w_2, w_3$) while avoiding restricted operating zones (labeled as "$\textnormal{ROZ}_1$, $\textnormal{ROZ}_2$") and enemy agents (purple and green UAVs). We asked the participants to visit the regions in a sequence, i.e., visiting the first region, then second, and then the third region. After visiting the third region, the task is to visit the first region again to perform the surveillance. 

For example, if the last visited region is $w_3$, then the safety specification in this scenario is $\pctlProb_{\geq \beta}((\neg  \texttt{crash} \enskip \wedge \enskip \neg \texttt{ROZ}) \enskip \pctlUntil\enskip \texttt{target})$, where \texttt{ROZ} is to visit the ROZ areas and \texttt{target} is visiting $w_1$.

We synthesize the autonomy protocol on the AMASE scenario with two enemy agents. The underlying MDP has 15625 states. We use the same blending function and same threshold $\beta=0.7$ as in the gridworld example. The features to compute the human strategy with MEIRL are given by the distance to the closest ROZ, enemy agents, and the target region. 

The human strategy $\sigma_h$ violates the specification with a probability of 0.496.
Again, we compute the repaired strategy $\sigma_{ha}$ with the greedy and the QCP approach. Both strategies satisfy the specification.
On the one hand, the maximum deviation between $\sigma_h$ and $\sigma_{ha}$ is 0.418 with the greedy approach, which means the strategies of the human and the autonomy protocol are significantly different in some states of the MDP. 
On the other hand, the QCP approach yields a repaired strategy $\sigma_{ha}$ that is more similar to the human strategy $\sigma_{h}$ with a maximum deviation of 0.038. 
The time of the synthesis procedure with the LP approach is 481.31 seconds and the computation time with the QCP approach is 749.18 seconds, showing the trade-offs between the greedy approach and the QCP approach. We see that the greedy approach can compute a feasible solution slightly faster, however the resulting blended strategy may be less similar to the human strategy compared to the QCP approach. 

\begin{table}[t]
\centering
\caption{Results for different specification thresholds for the probability and expected time in the AMASE example. '$\beta$' and '$\kappa$' refer to the threshold for the probability and the expected time of the specification.}
\label{tab:cost}
\scalebox{1}{\begin{tabular}{@{}ccccc@{}}
\toprule
$\beta$         & $\kappa$     & \vtop{\hbox{\strut Synthesis time with the}\hbox{\strut \enskip\thinspace  QCP approach (sec)}}  & $\delta_{\textnormal{QCP}}$ \\
\midrule
$0.7$    &  $20$   &  $827.37$ & $0.380$  \\
$0.7$    &  $40$   &  $749.14$ & $0.126$  \\
$0.7$    &  $80$   &  $722.81$ & $0.054$  \\ 
$0.9$    &  $20$   &  $888.29$ & $0.598$  \\   
$0.9$    &  $40$   &  $795.98$ & $0.163$  \\     
$0.9$    &  $80$   &  $732.41$ & $0.100$  \\          
\bottomrule\end{tabular}}
\end{table}

To assess the effect of changing the threshold of satisfying the specification, we use a different threshold $\beta=0.9$. The greedy approach did not terminate within one hour, and could not find a repaired strategy that satisfies the specification after 45 iterations. We compute a repaired strategy $\sigma_{ha}$ using the QCP approach with a maximum deviation of 0.093. The computation time with the QCP approach is 779.81 seconds, showing that the QCP approach does not take significantly more time to compute a repaired strategy even with a higher threshold. We conclude that the greedy approach may not be able to find a feasible strategy efficiently if most of the strategies in an MDP do not satisfy the specification.

We also assess the effect of adding additional constraints to the task, i.e., surveilling the next green region within a certain time step. We synthesize different policies for different expected times until the UAV reaches the next region. We summarize the results in Table~\ref{tab:cost}. For each different probability thresholds (labeled as "$\beta$") and expected times to complete the mission (labeled as "$\kappa$"), we report the synthesis time and the maximal deviation. The results in Table~\ref{tab:cost} illustrate that the maximal deviation $\delta_{\textnormal{QCP}}$ increases with increasing threshold and decreasing expected time to complete the mission. For example, with the threshold $\beta=0.9$ and expected time $\kappa=20$, the maximal deviation between the human and the repaired strategy is $0.598$, which shows that the strategies of the human and the autonomy protocol can be significantly different in some states. On the other hand, with the threshold $\beta=0.7$ and expected time $\kappa=80$, the maximal deviation between the human strategy and the repaired strategy is $0.054$, which is significantly smaller than the previous examples. We also note that there is no significant difference in synthesis time for different thresholds and expected times.

\begin{table}[t]
\centering
\caption{Results for different perturbations of the human strategy in the AMASE example. '$\delta_{\textnormal{max}}$' refers to the maximal perturbation introduced to the human strategy. $\delta_{\textnormal{ap}}$ refers to the maximal deviation between the repaired strategy and the human strategy.}
\label{tab:perturbation}
\scalebox{1}{\begin{tabular}{@{}cccccc@{}}
\toprule
$\beta$         & $\delta_{\textnormal{max}}$     & $\delta_{\textnormal{ap}}$  & \vtop{\hbox{\strut Synthesis time with the}\hbox{\strut \enskip \thinspace QCP approach (sec)}}  & $\delta_{\textnormal{QCP}}$ \\
\midrule
$0.7$    &  $0.1$    & $0.170$   &  $725.93$ & $0.032$   \\
$0.7$    &  $0.2$    & $0.274$   &  $718.14$ & $0.036$	   \\ 
$0.7$    &  $0.5$    & $0.506$   &  $696.60$ & $0.037$   \\        
$0.9$    &  $0.1$    & $0.270$   &  $732.54$ & $0.091$   \\   
$0.9$    &  $0.2$    & $0.345$   &  $745.05$ & $0.092$   \\     
$0.9$    &  $0.5$    & $0.534$   &  $798.01$ & $0.101$   \\          
\bottomrule\end{tabular}}
\end{table}

\subsection{Effect of changing the human strategy}
In this section, we investigate how changing the human strategy changes the strategy of the autonomy protocol. We perturb the human strategy from the previous example using~\eqref{eq:perturbationlp} with different perturbation functions $\delta$. We use three different values for maximal perturbation for every state and action between the human strategy and the repaired strategy and two different thresholds to satisfy the specification with $\beta=0.7$ and $\beta=0.9$.

We summarize our results in Table~\ref{tab:perturbation}. We generated three different perturbed human strategies with perturbation functions that have a different maximal perturbation (labeled as "$\delta_{\textnormal{max}}$"). We report the maximal deviation between the autonomy protocol that is synthesized using the original human strategy and the perturbed human strategy (labeled as $\delta_{\textnormal{ap}}$), the time that the synthesis process took with the QCP approach (labeled as "QCP synth."), and the maximal deviation between the perturbed human strategies and the repaired strategies (labeled as $\delta_{\textnormal{QCP}}$). 

The results in Table~\ref{tab:perturbation} show that the maximal deviation between the repaired strategy and the human strategy does not depend on the perturbation, and it depends on the threshold of satisfying the specification. The maximal deviation between the repaired strategies increases with larger perturbations introduced to the human strategy and with a larger threshold $\beta$. The values for $\delta_{\textnormal{QCP}}$ show that the maximal deviation between the human strategy and the repaired strategy does not depend heavily on a specific human strategy, and it mostly depends on the threshold. We also note that the synthesis time is similar for all cases.


\section{Conclusion and Critique}\label{sec:conclusion}
We introduced a formal approach to synthesize an autonomy protocol in a shared control setting subject to probabilistic temporal logic specifications. The proposed approach utilizes inverse reinforcement learning to compute an abstraction of a human's behavior as a randomized strategy in a Markov decision process. We designed an autonomy protocol such that the resulting robot strategy satisfies safety and performance specifications. 
We also ensured that the resulting robot behavior is as similar to the behavior induced by the human's commands as possible. 
We synthesized the robot behavior using quasiconvex programming.
We showed the practical usability of our approach through case studies involving autonomous wheelchair navigation and unmanned aerial vehicle planning. 

There is a number of limitations and also possible extensions of the proposed approach. First of all, we computed a globally optimal strategy by bisection, which requires checking feasibility of a number of linear programming problems. A direct convex formulation of the shared control synthesis problem would make computing the globally optimal strategy more efficient.

We assumed that the human's commands are consistent through the whole execution, \ie, the human issues each command to satisfy the specification. Also, this assumption implies the human does not consider assistance from the robot while providing commands - and in particular, the human does not adapt the strategy to the assistance. It may be possible to extend the approach to handle non-consistent commands by utilizing additional side information, such as the task specifications.

Finally, in order to generalize the proposed approach to other task domains, it is worth to explore transfer learning~\cite{pan2010survey} techniques. Such techniques will allow us to handle different scenarios without requiring to relearn the human strategy from the human's commands.

\bibliographystyle{plain}
\bibliography{literature}


\begin{IEEEbiography}[{\includegraphics[width=1in,height=1.25in,clip,keepaspectratio]{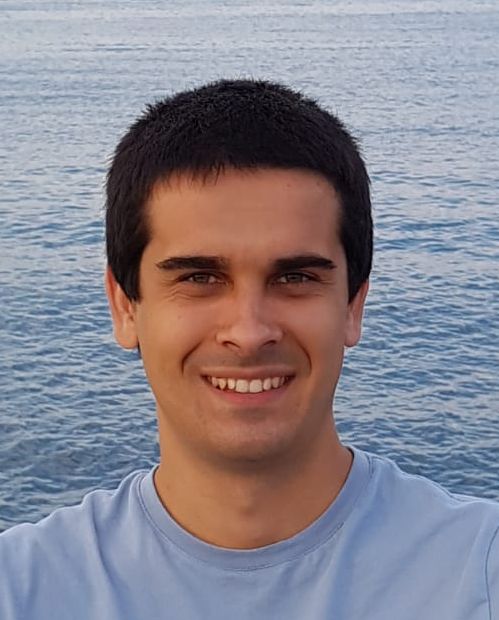}}]{Murat Cubuktepe} joined the Department of Aerospace Engineering at the University of Texas at Austin as a Ph.D. student in Fall 2015. He received his B.S degree in Mechanical Engineering from Bogazici University in 2015 and his M.S degree in Aerospace Engineering and Engineering Mechanics from the University of Texas at Austin in 2017. His current research interests are verification and synthesis of parametric and partially observable probabilistic systems. He also focuses on applications of convex optimization in formal methods and controls.
\end{IEEEbiography}

\begin{IEEEbiography}[{\includegraphics[width=1in,height=1.25in,clip,keepaspectratio]{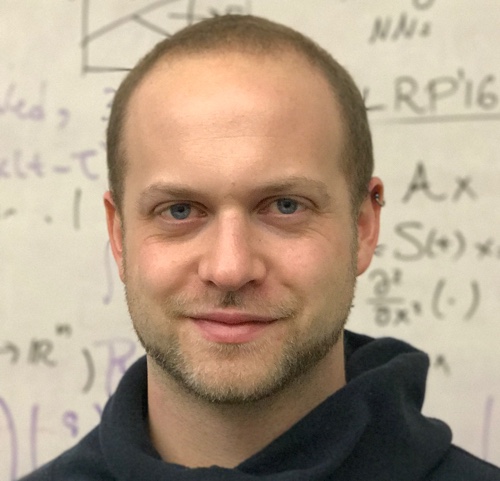}}]{Nils Jansen} is an assistant professor with the Institute for Computing and Information Science (iCIS) at the Radboud University, Nijmegen, The Netherlands. He received his Ph.D. in computer science with distinction from RWTH Aachen University, Germany, in 2015. Prior to Radboud University, he was a postdoctoral researcher and research associate with the Institute for Computational Engineering and Sciences at the University of Texas at Austin. His current research focuses on formal reasoning about safety aspects in machine learning and robotics. At the heart is the development of concepts inspired from formal methods to reason about uncertainty and partial observability.
\end{IEEEbiography}

\begin{IEEEbiography}[{\includegraphics[width=1in,height=1.25in,clip,keepaspectratio]{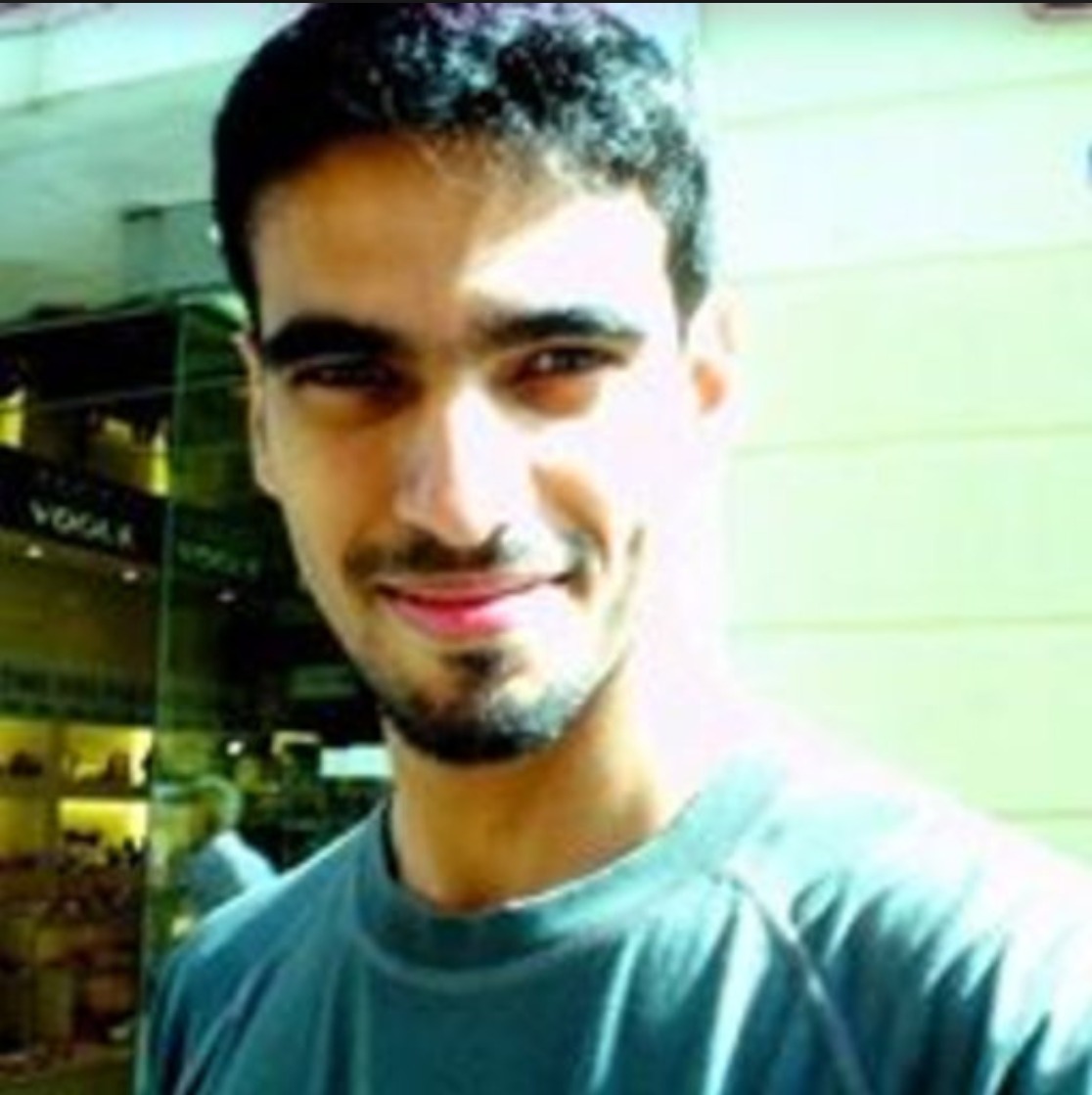}}]{Mohammed Alshiekh} was a research assistant in the Department of Aerospace Engineering at the University of Texas at Austin from 2016 to 2018. He received his BEng degree in Electrical and Electronics Engineering from the University of Birmingham in 2008 and his M.S degree in Systems Engineering from the University of Pennsylvania in 2016.
\end{IEEEbiography}

\vspace{-15cm}

\begin{IEEEbiography}[{\includegraphics[width=1in,height=1.25in,clip,keepaspectratio]{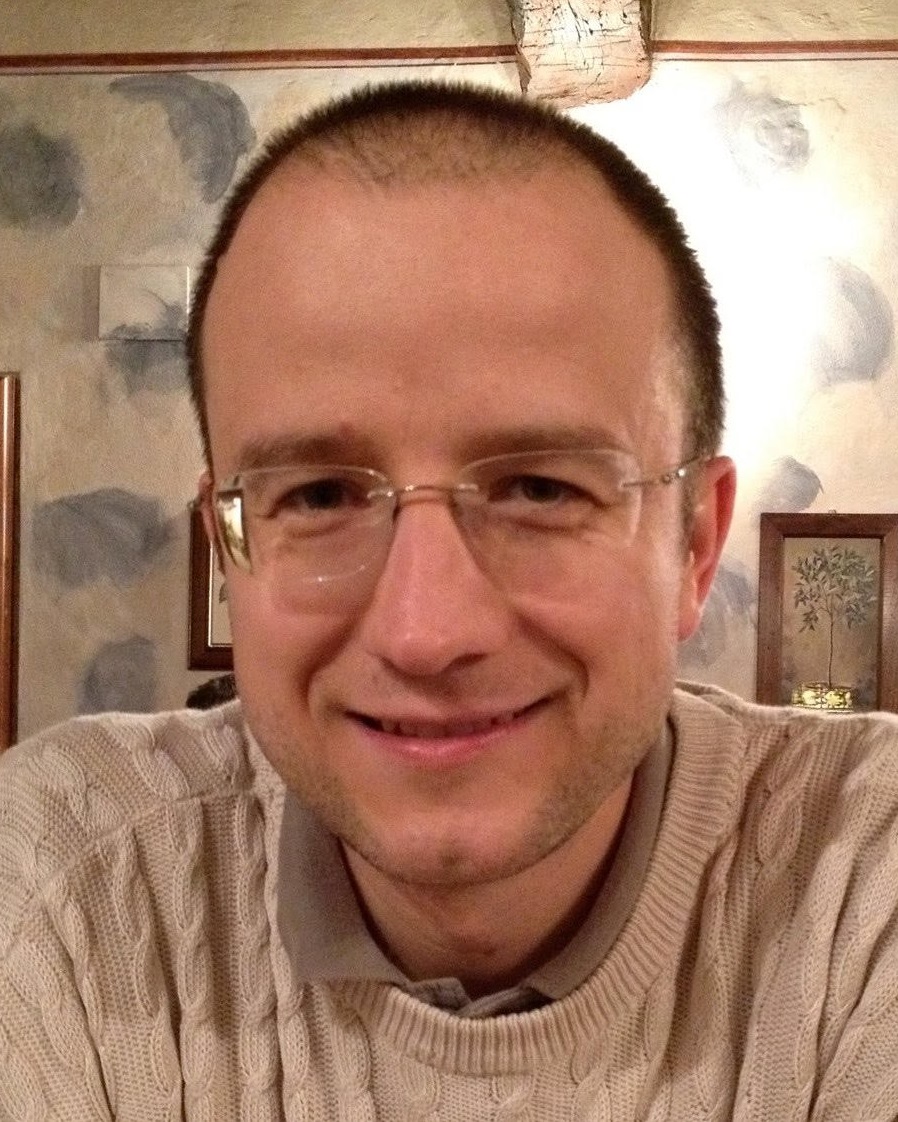}}]{Ufuk Topcu} joined the Department of Aerospace Engineering at the University of Texas at Austin as an assistant professor in Fall 2015. He received his Ph.D. degree from the University of California at
Berkeley in 2008. He held research positions at the University of
Pennsylvania and California Institute of Technology. His research
focuses on the theoretical, algorithmic and computational aspects of
design and verification of autonomous systems through novel
connections between formal methods, learning theory and controls.
\end{IEEEbiography}

%


\end{document}